\documentclass[runningheads]{llncs}

\usepackage{graphicx} 
\graphicspath{ {./images/} }
\usepackage[rightcaption]{sidecap}
\usepackage{wrapfig}
\usepackage{amsmath}
\usepackage{array,multirow}
\usepackage{makecell}
\usepackage{rotating}
\usepackage{framed}

\usepackage{algorithm}
\usepackage{algorithmic}
\usepackage[locale=FR]{siunitx}

\usepackage{graphicx}
\graphicspath{ {images/} }

\usepackage{listings}
\newcommand{\cX}{{\cal X}}
\newcommand{\cS}{{\cal S}}

\newcommand{\cM}{{\cal M}}

\newcommand{\cP}{{\cal P}}

\newcommand{\cI}{{\cal I}}

\newcommand{\Min}{\mathrm{Min}}

\newcommand{\size}{\mathrm{size}}
\newcommand{\wt}{\mathrm{wt}}

\newcommand{\al}{\allowbreak}
\newcommand{\mn}[1]{}

\begin{document}

\title{On Exact Learning of $d$-Monotone Functions}
\titlerunning{Learning $d$-Monotone Functions}

\author{Nader H. Bshouty}
\authorrunning{N. H. Bshouty} 
%
\tocauthor{N. H. Bshouty}
\institute{Technion}

\maketitle              

\begin{abstract}
In this paper, we study the learnability of the Boolean class of $d$-monotone functions $f:\cX\to\{0,1\}$ from membership and equivalence queries, where $(\cX,\le)$ is a finite lattice. 
We show that the class of $d$-monotone functions that are represented in the form $f=F(g_1,g_2,\ldots,g_d)$, where $F$ is any Boolean function $F:\{0,1\}^d\to\{0,1\}$ and $g_1,\ldots,g_d:\cX\to \{0,1\}$ are any monotone functions, is learnable in time $\sigma(\cX)\cdot (\size(f)/d+1)^{d}$ where $\sigma(\cX)$ is the maximum sum of the number of immediate predecessors in a chain from the largest element to the smallest element in the lattice $\cX$ and  $\size(f)=\size(g_1)+\cdots+\size(g_d)$, where $\size(g_i)$ is the number of minimal elements in $g_i^{-1}(1)$. 

For the Boolean function $f:\{0,1\}^n\to\{0,1\}$, the class of $d$-monotone functions that are represented in the form $f=F(g_1,g_2,\ldots,g_d)$, where $F$ is any Boolean function and $g_1,\ldots,g_d$ are any monotone DNF, is learnable in time $O(n^2)\cdot (\size(f)/d+1)^{d}$ where $\size(f)=\size(g_1)+\cdots+\size(g_d)$. 

In particular, this class is learnable in polynomial time when $d$ is constant. Additionally, this class is learnable in polynomial time when $\size(g_i)$ is constant for all $i$ and $d=O(\log n)$. 
\end{abstract}
\section{Introduction}
Let $\cP=(\cX,\le)$ be a lattice. A Boolean function $f:\cX\to \{0,1\}$ is \textit{$d$-monotone} if, for any chain $x_1<x_2<\cdots<x_t$ in $\cX$, the sequence $0,f(x_1),f(x_2),\ldots,f(x_t)$ changes its value at most $d$ times. If $d=1$, we say that $f$ is a \textit{monotone} function.

In this paper, we study the learnability of $d$-monotone functions. The first fact that motivates the study of this class is that every Boolean function is $d$-monotone for some $d\le n$. 
The second is Markov's result~\cite{Markov58}, which states: The minimum number of negation gates in an AND-OR-NOT
circuit that computes $f$ is $\log d+O(1)$ if and only if $f$ is an $O(d)$-monotone function. Therefore, learning
$d$-monotone functions can be seen as similar to learning functions with few negations~\cite{BlaisCOST15}.

When $\cX=\{0,1\}^n$, the problem of learning monotone and $d$-monotone Boolean functions has been extensively studied in the 
literature. See~\cite{AmanoM06,Angluin87,Black,BlaisCOST15,BlumBL98,Bshouty97,BshoutyT96,GoldreichGLRS00,HarmsY22} \cite{LangeRV22,LangeV23,ODonnellS07,ODonnellW09,Servedio04,TakimotoSM00}.

In the PAC learning without membership queries under the uniform distribution, Bshouty-Tamon \cite{BshoutyT96} and Lange et al.~\cite{LangeRV22,LangeV23} proved that monotone functions can be learned in time $\exp{(\sqrt{n}/\epsilon)}$. 
Blais et al.~\cite{BlaisCOST15} extended the result to $d$-monotone functions. They provided an algorithm that runs in time $\exp{(d\sqrt{n}/\epsilon)}$ and showed that this algorithm is optimal. See also~\cite{Black}.

In the exact learning with membership and equivalence queries, Angluin~\cite{Angluin87} proved that any monotone DNF $f$ can be learned in polynomial time (poly$(n,\al\size(f))$) with $\size(f)$ equivalence queries and $n\cdot\size(f)$ membership queries, where $\size(f)$ is the number of monotone terms (minterms) in $f$. 
One possible representation of $d$-monotone function introduced by Blais et al.~\cite{BlaisCOST15} uses the fact that every $d$-monotone function can be expressed as $g_1\oplus g_2\oplus \cdots\oplus g_d$, where each $g_i$ is a monotone DNF, and $\oplus$ denotes the exclusive OR (XOR) operation.
Takimoto et al.~\cite{TakimotoSM00} show that if $g_d\Rightarrow g_{d-1}\Rightarrow \cdots\Rightarrow g_1$ and for every $i\le d-1$, there is no term that appears in both\footnote{Takimoto et al. claim that their result applies for any $g_i$ that satisfies $g_d\Rightarrow g_{d-1}\Rightarrow \cdots\Rightarrow g_1$ and for every $i\le d-1$, $g_i\not=g_{i+1}$. 
In this paper, we show that this claim is not entirely accurate. For their algorithm to be valid, it is necessary that for every $i\le d-1$, no term appears in both $g_i$ and $g_{i+1}$. See also~\cite{GuijarroLR01} page~560.} $g_i$ and $g_{i+1}$, then $f$ is learnable from at most $n\prod_i\size(g_i)\le n(\size(f)/d+1)^d$ equivalence queries and  $n^3\prod_i\size(g_i)\le n^3(\size(f)/d+1)^d$ membership queries, where $\size(f)=\size(g_1)+\cdots+\size(g_d)$.

This paper studies the learnability of the $d$-monotone function in a very general representation. We study the class of $d$-monotone functions represented in the form $F(g_1,g_2,\al\ldots,g_d)$ where $F$ is {\it any} Boolean function $F:\{0,1\}^d\to \{0,1\}$ and each $g_i$ is {\it any} monotone DNF. 

We first state the result in the general setting when $g_i:\cX\to\{0,1\}$ where $\cX$ is any lattice.

\begin{theorem}\label{THEO1}
    Let $(\cX,\le)$ be a finite lattice. 
    The class of $d$-monotone functions $f:\cX\to\{0,1\}$, that are represented in the form $f=F(g_1,g_2,\ldots,g_d)$, where $F$ is any Boolean function $F:\{0,1\}^d\to\{0,1\}$ and $g_1,\ldots,g_d:\cX\to \{0,1\}$ are any monotone functions, is learnable in time $\sigma(\cX)\cdot (\size(f)/d+1)^{d}$ where $\sigma(\cX)$ is the maximum sum of the number of immediate predecessors in a chain from the largest element to the smallest element in the lattice $\cX$ and  $\size(f)=\size(g_1)+\cdots+\size(g_d)$, where $\size(g_i)$ is the number of minimal elements in $g_i^{-1}(1)$. 

    The algorithm asks at most $(\size(f)/d+1)^{d}$ equivalence queries and $\sigma(\cX)\cdot (\size(f)/d+1)^{d}$ membership queries.
\end{theorem}

For the lattice $\{0,1\}^n$ with the standard $\le$, we have $\sigma(\{0,1\}^n)=n(n+1)/2=O(n^2)$ and therefore,
\begin{corollary}\label{Coroll}
    The class of $d$-monotone functions $f:\{0,1\}^n\to \{0,1\}$ that are represented in the form $f=F(g_1,g_2,\ldots,g_d)$, where $F$ is any Boolean function and $g_1,\ldots,g_d$ are any monotone DNF, is learnable in time $O(n^2)\cdot (\size(f)/d+1)^{d}$, where $\size(f)=\size(g_1)+\cdots+\size(g_d)$.  

    The algorithm asks at most $(\size(f)/d+1)^{d}$ equivalence queries and $n^2\cdot (\size(f)/d+1)^{d}$ membership queries.
\end{corollary}

In particular, the following classes are learnable in polynomial time ($poly(\al\size(f),n)$):
\begin{enumerate}
    \item The class of $d$-monotone functions where $d$ is constant.
    \item The class of $O(\log n)$-monotone functions of size $\size(f)=O(\log n)$.
\end{enumerate}

To compare our result with Takimoto et al. \cite{TakimotoSM00}, we prove that there is a function $f$ that can be represented as $f=F(g_1,\ldots,g_d)$ and has size $s$, where the representation $f=G_1\oplus G_2\oplus \cdots \oplus G_d$ of Takimoto et al. is of size at least $O(s^d)$. 
This, by their analysis, implies that for $f$, their algorithm asks $O(ns^{d^2})$ equivalence queries and $O(n^3s^{d^2})$ membership queries, while our algorithm asks at most $O(s^d)$ equivalence queries and $O(n^2s^{d})$ membership queries.

\section{Definitions and Preliminary Results}
Let $\cX$ be a finite set. Let $\cP=(\cX,\le)$ \mn{$\cP=(\cX,\le)$} be a lattice. We say that $b$ is an \textit{immediate predecessor} \mn{immediate predecessor} of $a$ if $b<a$ and there is no $c$ such that $b<c<a$. 
We say that $a,b\in \cX$ are {\it incomparable} if neither $a\le b$ nor $b\le a$ holds. Otherwise, they are {\it comparable}. 
The\footnote{In a lattice, the join exists, and it is unique.} {\it join} \mn{\\ join} $a\vee b$ of $a$ and $b$ is the smallest element in $\cX$ that is greater than or equal to both $a$ and $b$. 
For two sets $X_1,X_2\subseteq \cX$, we define the {\it join} of $X_1$ and $X_2$ as $X_1\vee X_2=\{x_1\vee x_2~|~x_1\in X_1,x_2\in X_2\}$. We say that $a$ is a {\it minimal} element \mn{minimal element in $\cS$} in $\cS\subset\cX$ if no element in $\cS$ is smaller than $a$. 
We denote by $\Min (\cS)$ \mn{\\ $\Min (\cS)$} the set of all minimal elements in $\cS$. A \textit{chain} is a totally ordered subset of $\cX$. That is, $C\subset\cX$ is a chain if every pair of elements in $C$ is comparable. 

We define {\it the maximal predecessor sum} $\sigma(\cX)$ as the maximum sum of the number of immediate predecessors in a chain from the largest element to the smallest element in a lattice $\cX$. Formally, let $m$ be the largest element of $(\cX,\le)$, and let $X=\{x_1,\ldots,x_r\}$ be its set of immediate predecessors. Define the sub-lattice $(\cX_i,\le)$, where $\cX_i=\{x\in \cX|x\le x_i\}$, with $x_i$ as the largest element. Then, 
$$\sigma(\cX)=|X|+\max_{i\in[r]} \sigma(\cX_i)$$
where $\sigma$ of a singleton set is defined as $0$. 

We will add to the lattice $\cP$ a minimum element $\perp\not\in \cX$ such that $\perp<x$ for all $x\in \cX$. This will ease the analysis and the proofs, which are all true without this element.

When $\cX=\{0,1\}^n$, for two elements $x,y\in \{0,1\}^n$, we define $x\le y$ if and only if $x_i\le y_i$ for all $i\in [n]$. The join $x\vee y$ of $x$ and $y$ is the bitwise OR of $x$ and $y$. It is easy to see that $(\{0,1\}^n,\le)$ is a lattice. 

\subsection{The Model}
The learning criterion we consider is {\it exact learning model}. There is a function \( f:\cX\to\{0,1\}\), called the {\it target function},  
which belongs to a class of functions \( C \). The goal of the learning  
algorithm is to halt and output a formula \( h \) that is  
logically equivalent to \( f \).  

In a {\it membership query}, the learning algorithm supplies an assignment \( a\in \cX \) 
as input to a membership oracle and receives in return  
the value of \( f(a) \).  
In an {\it equivalence query}, the learning algorithm supplies any function \( h:\cX\to\{0,1\} \) as input to an equivalence oracle,  
and the oracle's response is either  
$ \text{``yes''} $  
indicating that \( h \) is equivalent to \( f \), or a {\it counterexample},  which is an assignment \( b \) such that  $h(b) \neq f(b).$

\subsection{Monotone Functions}
In this section, we define the concept of monotone functions and give some results.

Let $a\in \cX$ and $M_a:\cX\to\{0,1\}$ be the function defined by $M_a(x)=1$ if and only if $x\ge a$. 
We call $M_a$ a \textit{monotone term} \mn{monotone term} that is generated by $a$. 
A {\it monotone function} $f$ is a disjunction of monotone terms. 
If $f$ is a monotone function, then $f$ is the disjunction of the monotone terms generated by the elements of $\Min(f^{-1}(1))$. 
Thus,
$$f=\bigvee_{a\in\Min(f^{-1}(1))}M_a.$$
We will denote $\Min(f)=\Min(f^{-1}(1))$. \mn{$\Min(f)$}

The following is a well-known result.
\begin{lemma}\label{monotoneD}
    The function $f:\cX\to\{0,1\}$ is monotone if and only if for every $x\ge y$, we have $f(x)\ge f(y)$.
\end{lemma}
The \textit{size} of the monotone function $\size(f)$ is defined as $|\Min(f)|=|\Min(f^{-1}(1)|$. 
The elements of~$\Min(f)$ are called the \textit{minimal elements} \mn{minimal element of $f$} of~$f$, and $M_a$, $a\in \Min(f)$, are called the \textit{minterms of~$f$}.\mn{\\ \ \\minterm of $f$.} 
It is easy to see that the minimal elements of a monotone function are incomparable. 

For any Boolean function $f:\cX\to \{0,1\}$, we define $f(\perp)=0$. 

The following result is easy to prove.
\begin{lemma}\label{MEqui}
    Let $f:\cX\to \{0,1\}$ be a monotone function. The element $a$ is a minimal element of $f$ if and only if $f(a)=1$, and for every immediate predecessor $b$ in $\cX\cup\{\perp\}$ of~$a$, we have $f(b)=0$. 
\end{lemma}

We now prove
\begin{lemma}\label{MinORAND}
    For any two monotone functions $g$ and $h$, we have:
    \begin{enumerate}
        \item\label{MinORAND1} $\Min(g\vee h)\subseteq \Min(g)\cup \Min(h)$.
        \item\label{MinORAND2} $\Min(g\wedge h)\subseteq \Min(g)\vee \Min(h)$.
        \item\label{MinORAND3} $u=v\vee w$ if and only if $M_u=M_v\wedge M_w$.
    \end{enumerate}
\end{lemma}
\begin{proof}
    To prove item~\ref{MinORAND1}, we use Lemma~\ref{MEqui}. 
    Let $a$ be a minimal element of $g\vee h$. 
    Then $g(a)\vee h(a)=1$ and therefore $g(a)=1$ or $h(a)=1$. 
    For any immediate predecessor $b$ of $a$, we have $g(b)\vee h(b)=0$ which implies that $g(b)=0$ and $h(b)=0$. 
    Therefore $a\in \Min(g)\cup \Min(h)$.

    We now prove item~\ref{MinORAND2}. 
    Let $a$ be a minimal element of $f=g\wedge h$. 
    Then $g(a)\wedge h(a)=1$, and therefore $g(a)=1$ and $h(a)=1$. Let $u$ be a minimal element of $g$ such that $u\le a$ and $w$ be a minimal element of $h$ such that $w\le a$. 
    We now show that $a=u\vee w$. Suppose to the contrary that $a'=u\vee w<a$. 
    Since $a'>u,w$, by Lemma~\ref{monotoneD}, we have $g(a')=1$ and $h(a')=1$. Therefore, $f(a')=1$. Since $a'<a$, and $f(a')=1$, we have $a\not\in \Min(f)$. This is a contradiction. Therefore, $a=u\vee w\in \Min(g)\vee \Min(f)$. 

    We now prove item~\ref{MinORAND3}. ($\Leftarrow$). If $M_u=M_v\wedge M_w$, then by item~\ref{MinORAND2}, we have
    $\{u\}=\Min(M_u)\subseteq \Min(M_v)\vee \Min(M_w)=\{v\vee w\}$. Therefore, $u=v\vee w$.

    ($\Rightarrow$). Now, if $u=v\vee w$, then $M_u(x)=1$ iff $x\ge u=v\vee w$ iff $x\ge v$ and  $x\ge w$ iff $M_v(x)=1$ and $M_w(x)=1$ iff $M_v(x)\wedge M_w(x)=1$.\qed
\end{proof}

\subsection{$d$-Monotone Functions}
This section defines the concept of $d$-monotone functions and proves some results. 

Recall that\footnote{This definition is for any Boolean function $f$. So, $\overline{f}(\perp)=0$, where $\overline{f}$ denotes the negation of $f$.} $f(\perp)=0$.
\begin{definition}
  Let $f:\cX\to\{0,1\}$ be a Boolean function.  We say that $f$ is {\it $d$-monotone} 
if, along any chain $\perp<x_1<x_2<\cdots<x_t$ in $\cX\cup\{\perp\}$, the function changes its value at most $d$ times.   
\end{definition}

It is easy to see that $f$ is monotone if and only if it is $1$-monotone or $0$-monotone ($f=0$). 

We now prove,
\begin{lemma}\label{Rep01}
    Let $g_1,\ldots,g_d:\cX\to\{0,1\}$ be non-constant monotone Boolean functions and $F:\{0,1\}^d\to \{0,1\}$ be any Boolean function. Then\footnote{Note here that $f(\perp)=0$ and may not necessarily be equal to $F(g_1(\perp),\ldots,g_d(\perp))=F(0,0,\ldots,0)$.} $f=F(g_1,\ldots,g_d)$ is $(d+1)$-monotone. 

    If $F(0^d)=0$, then $f$ is $d$-monotone. 
\end{lemma}

\begin{proof}
    Let $C:\perp<x_1<x_2<\cdots<x_t$ be any chain in $\cX\cup\{\perp\}$. 
    Suppose $g_i$ changes its value from $0$ to $1$ along this chain at $x_{j_i}$ and assume, without loss of generality, that $j_1\le j_2\le \cdots\le j_d$. 
    Then for the elements $\{x_i|1\le i\le j_1-1\}$, the value of the function $f$ is equal to $F(0,0,\ldots,0)$, and for the elements $\{x_i|j_1\le i\le j_2-1\}$, the function $f$ is equal to $F(1,0,\cdots,0)$, and for the elements $\{x_i|j_2\le i\le j_3-1\}$, the function $f$ is equal to $F(1,1,0,\cdots,0)$, etc. 
    That is, the function along the chain $x_1<x_2<\cdots<x_t$ changes its values only on a subset of $\{x_{j_1},x_{j_2},\ldots,x_{j_d}\}$. Since $f(\perp)=0$ (by definition) and this may be not equal to $F(g_1(\perp),\ldots,g_d(\perp))=F(0,0,\ldots,0)$, the function along the chain $C$ changes its values only on a subset of $\{x_1,x_{j_1},x_{j_2},\ldots,x_{j_d}\}$. Therefore, it is $(d+1)$-monotone.

    If $F(0,0,\ldots,0)=0=f(\perp)$, then the function along the chain changes its values only on a subset of $\{x_{j_1},x_{j_2},\ldots,x_{j_d}\}$. Therefore, it is $d$-monotone.\qed
\end{proof}

We note here that for the purpose of learning, we can assume that $F(0^d) = 0$. This is because, if $F(0^d) = 1$, then we can learn $F' = F \oplus 1$ which satisfies $F'(0^d) = 0$, and then recover $F$ as $F=F'\oplus 1$.

\subsection{Minimal Elements of a Function}
In this section, we extend the definition of minimal element to any Boolean function. Since Lemma~\ref{MEqui} is not necessarily true for non-monotone functions, we must define two types of minimal elements: local and global.

For any Boolean function $f:\cX\to\{0,1\}$, we say that $a$ is a {\it local minimal element}\mn{local minimal element of $f$} of $f$ if $f(a)=1$ and for every immediate predecessor $b$ of $a$, $f(b)=0$. 
We denote by $\min (f)$\mn{\\ $\min (f)$} the set of all local minimal elements of $f$. 
We say that $a$ is a {\it global minimal element}\mn{\\ global minimal element of $f$} of $f$ if $f(a)=1$ and for every $b<a$ we have $f(b)=0$. We denote by $\Min(f)$\mn{\\ \ \\ $\Min(f)$} the set of all global minimal elements of $f$. Obviously, every global minimal element of $f$ is also a local minimal element of $f$, and therefore
$$\Min(f)\subseteq \min(f).$$

When the function $f$ is monotone, by Lemma~\ref{monotoneD} and Lemma~\ref{MEqui}, $\Min(f)=\min(f)$. 

We now prove
\begin{lemma}\label{BaseF} 
    Let $F:\{0,1\}^d\to\{0,1\}$ where $F(0^d)=0$. Let $f=F(g_1,g_2,\ldots,g_d)$ where $g_1,g_2,\ldots,g_d$ are monotone functions. Then 
    $$\min(f)\subseteq \bigcup_{I\subseteq [d]} \left( \Min\left(\bigwedge_{i\in I}g_i\right)\right)\subseteq \bigcup_{I\subseteq [d]} \left(\bigvee_{i\in I} \Min(g_i)\right).$$

    If $g_d\Rightarrow g_{d-1}\Rightarrow \cdots \Rightarrow g_1$ then 
    $$\min(f)\subseteq \bigcup_{i=1}^d  \Min(g_i).$$
\end{lemma}
\begin{proof}
    Let $a$ be a local minimal element of $f$. Then $f(a)=1$ and for every immediate predecessor $b$ of $a$, we have $f(b)=0$. If $g_i(a)=0$ for all $i\in [d]$, then $f(a)=F(0^d)=0$. Therefore, there is some $i$ such that $g_i(a)=1$.   

    Let $I\subseteq [d]$ be such that $g_i(a)=1$ for all $i\in I$ and $g_i(a)=0$ for all $i\not\in I$. Let $h=\wedge_{i\in I}g_i$. Then $h(a)=1$. Let $b$ be any immediate predecessor of $a$. Since $b<a$, and $g_i$ are monotone, $g_i(b)=0$ for every $i\not\in I$. Since $f(a)=1\not=0=f(b)$, we must have $g_i(b)=0$ for some $i\in I$. Therefore, $h(b)=0$. Thus, $a$ is a minimal element of $h=\wedge_{i\in I}g_i$, and by Lemma~\ref{MinORAND}, $a\in \vee_{i\in I}\Min(g_i)$.

    If $g_d\Rightarrow g_{d-1}\Rightarrow \cdots\Rightarrow g_1$, then $h=\wedge_{i\in I}g_i=g_j$ for $j=\max I$, and then $a\in \Min(g_j)$. \qed
\end{proof}

\subsection{The Minimum Monotone Closure of a Function}
In this section, we introduce the minimum monotone closure of a function as defined in~\cite{Bshouty95} and the strict monotone representation of a Boolean function as defined in~\cite{TakimotoSM00}, and show how to use them for $d$-monotone functions.

Let $f:\cX\to \{0,1\}$ be any function. We define the \textit{minimum monotone closure} of $f$ (or simply the \textit{monotone function} of $f$)\mn{Monotone function of $f$}, $\cM(f):\cX\to\{0,1\}$ to be the function that satisfies $\cM(f)(x)=1$ if there is $y\le x$ such that $f(y)=1$. The following is trivial; see, for example,~\cite{Bshouty95}.
\begin{lemma}
    We have
    \begin{enumerate}
        \item $\cM(f)$ is the minimum monotone function\footnote{Here, ``minimum'' means that for any other monotone function $g$, if $f\Rightarrow g$, then $\cM(f)\Rightarrow g$.} that satisfies $f\Rightarrow \cM(f)$. In particular,
        \item If $f(a)=1$, then $\cM(f)(a)=1$, and if $\cM(f)(b)=0$, then $f(b)=0$.
        \item $\Min(\cM(f))=\Min(f)$.
    \end{enumerate}
\end{lemma}

The following lemma is proved in~\cite{TakimotoSM00} for any Boolean function when $d=n$. For $d$-monotone functions, we prove:
\begin{lemma}\label{MaimM}
Let $f$ be a $d$-monotone function. Define $f_{i+1}=f_{i}\oplus\cM(f_i)=\overline{f_i}\wedge \cM(f_i)$, where $f_1=f$. Then
$$f=\cM(f_1)\oplus \cM(f_2)\oplus \cdots \oplus \cM(f_d).$$
\end{lemma}
\begin{proof}
    We prove the result by proving the following items: 
    \begin{enumerate}
        \item \label{MaimM1}$\cM(f_{i+1})\Rightarrow \cM(f_i)$.
        \item \label{MaimM2}If $z\in \Min(\cM(f_i))$, then $\cM(f_i)(z)=1$ and $\cM(f_{i+1})(z)=0$. In particular, $\Min(\cM(f_i))\cap \Min(\cM(f_{i+1}))=\emptyset.$ 
        \item \label{MaimM3}There exists $m$ such that $\cM(f_i)(x)=0$ for all $i> m$ and all $x$. 
        \item \label{MaimM4} Let $g=\cM(f_1)\oplus \cM(f_2)\oplus \cdots \oplus \cM(f_m)$. If $z\in\Min(\cM(f_j))=\Min(f_j)$, then $g(z)=(j\mod 2)$.
        \item \label{MaimM5} Let $g=\cM(f_1)\oplus \cM(f_2)\oplus \cdots \oplus \cM(f_m)$. Then $f=g$.
        \item \label{MaimM6}If $f$ is $d$-monotone, then $g(x)=\cM(f_1)\oplus \cM(f_2)\oplus \cdots \oplus \cM(f_d)$.
    \end{enumerate}   
    We prove item~\ref{MaimM1}. If $\cM(f_{i+1})=0$, the result follows.
    If $\cM(f_{i+1})\not=0$, then let $z$ be any element in $\cX$ such that $\cM(f_{i+1})(z)=1$. Thus, there exist $y\le z$ such that $f_{i+1}(y)=1$. Since $1=f_{i+1}(y)=\overline{f_i(y)}\wedge \cM(f_i)(y)$, we have $\cM(f_i)(y)=1$. Since $\cM(f_i)$ is monotone and $z\ge y$, we also have $\cM(f_i)(z)=1$. 
    Therefore, $\cM(f_{i+1})\Rightarrow \cM(f_i)$. 
 
    We now prove item~\ref{MaimM2}. Let $z\in \Min(\cM(f_i))=\Min(f_i)$. Then $f_i(z)=1$ and $\cM(f_i)(z)=1$. Thus, $f_{i+1}(z)=f_i(z)\oplus\cM(f_i)(z)=0$. Since $z\in \Min(\cM(f_i))=\Min(f_i)$, for every $y< z$ we have $f_i(y)=0$ and $\cM(f_i)(y)=0$, and therefore for every $y\le z$ we have $f_{i+1}(y)=f_i(y)\oplus\cM(f_i)(y)=0$. Therefore, $\cM(f_{i+1})(z)=0$. 

    Items \ref{MaimM1} and \ref{MaimM2} imply that $\cM(f_{i+1})\Rightarrow \cM(f_i)$ and $\cM(f_{i+1})\not=\cM(f_i)$. This implies item~\ref{MaimM3}.

    We now show item~\ref{MaimM4}. 
    Let $z\in \Min(\cM(f_j))$. By item \ref{MaimM2}, we have $\cM(f_{j})(z)=1$ and $\cM(f_{j+1})(z)=0$. Therefore, by item~\ref{MaimM1}, $\cM(f_{i})(z)=0$ for all $i\ge j+1$ and $\cM(f_i)(z)=1$ for all $i\le j$. This implies the result.

     We now prove item~\ref{MaimM5}. Let $g=\cM(f_1)\oplus \cM(f_2)\oplus \cdots \oplus \cM(f_m)$.
    Let $x\in \cX$. If $\cM(f_1)(x)=0$, then $f(x)=f_1(x)=0$, and by item~\ref{MaimM1}, $\cM(f_i)(x)=0$ for all $i$, and therefore $f(x)=g(x)$. If $\cM(f_j)(x)=1$ and $\cM(f_{j+1})(x)=0$, then by item~\ref{MaimM1}, $\cM(f_i)(x)=1$ for all $i\le j$ and $\cM(f_{i})(x)=0$ for all $i>j$. Therefore, $g(x)=(j\mod 2)$. Since $\cM(f_{j+1})(x)=0$, we have $f_{j+1}(x)=0$. Since for $i\le j$, $f_{i+1}(x)=f_{i}(x)\oplus\cM(f_{i})(x)=f_i(x)\oplus 1$, we have $f_i(x)=f_{i+1}(x)\oplus 1$. Now, since $f_{j+1}(x)=0$, we get $f(x)=f_1(x)=(j\mod 2)$. Therefore $f(x)=g(x)$.

    To prove item~\ref{MaimM6}, it is enough to show that $\cM(f_{d+1})=0$. Assume to the contrary $\cM(f_{d+1})\not=0$. We construct a chain of $d+2$ elements in $\cX\cup\{\perp\}$ with alternating values in $f$ and get a contradiction. We start from $x_{d+1}$ a minimal element of $\cM(f_{d+1})$. By items~\ref{MaimM4} and~\ref{MaimM5}, $f(x_{d+1})=g(x_{d+1})=(d+1\mod 2)$. By item~\ref{MaimM2}, $x_{d+1}\not\in \Min(\cM(f_d))$ and since $\cM(f_{d+1})\Rightarrow \cM(f_{d})$, $\cM(f_{d})(x_{d+1})=1$ and therefore there is a minimal element $x_d<x_{d+1}$ of $\cM(f_d)$. By items~\ref{MaimM4} and~\ref{MaimM5}, $f(x_{d})=g(x_{d})=(d\mod 2)\not=f(x_{d+1})$, and so on.
    
    This constructs a chain $x_1<x_2<\cdots<x_{d+1}$ with alternating values in $f$. 
    Since $x_1\in \Min(\cM(f_1))=\Min(f_1)$, we have $f(x)=f_1(x)=1$. We now add $\perp$ at the beginning of the chain and get a chain where, along this chain, the value of $f$ is changed $d+1$ times. Therefore, $\cM(f_{d+1})=0$. \qed
\end{proof}

Obviously, this representation is unique. We call such representation the \textit{strict monotone representation}\mn{strict monotone representation of $f$} of $f$. 

The following lemma presents some properties of this representation.
\begin{lemma}\label{Fi}
        Let $f$ be $d$-monotone function and let $f=\cM(f_1)\oplus \cdots\oplus \cM(f_d)$ be the strict monotone representation of $f$. Then
\begin{enumerate}
    \item \label{Fi1} $\cM(f_d)\Rightarrow\cM(f_{d-1})\Rightarrow \cdots\Rightarrow \cM(f_1)$.
    \item\label{Fi2} $f_i=\cM(f_i)\oplus \cM(f_{i+1})\oplus\cdots\oplus \cM(f_d).$
    \item \label{Fi3} For $j>i$, we have $\Min(\cM(f_i))\cap \Min(\cM(f_j))=\emptyset$.
\end{enumerate}  
\end{lemma}
\begin{proof}
    Item~\ref{Fi1} is item~\ref{MaimM1} in the proof of Lemma~\ref{MaimM}.
    
    The proof of item~\ref{Fi2} is by induction. First, by Lemma~\ref{MaimM}, we have $f_1=f=\cM(f_1)\oplus \cdots\oplus\cM(f_d)$. Then, by the induction hypothesis, we have
    $$f_{i+1}=f_i\oplus\cM(f_i)=\cM(f_i)\oplus \cM(f_{i+1})\oplus\cdots\oplus \cM(f_d)\oplus\cM(f_i)$$ $$=\cM(f_{i+1})\oplus\cdots\oplus \cM(f_d).$$

    To prove item~\ref{Fi3}, suppose to the contrary $a\in \Min(\cM(f_i))\cap \Min(\cM(f_j))$. Since $\cM(f_j)\Rightarrow \cM(f_{i+1})\Rightarrow \cM(f_i)$, it follows that $a\in \Min(\cM(f_{i+1})$. This contradicts item~\ref{MaimM2} in the proof of Lemma~\ref{MaimM}.\qed
\end{proof}

\section{The Algorithm}
In this section, we first provide a procedure that builds the hypothesis to the equivalent query. Then we present the algorithm that learns any $d$-monotone function of the form $F(g_1,\ldots,g_d)$, where $F:\{0,1\}^d\to \{0,1\}$ and each $g_i:\cX\to \{0,1\}$ is any monotone Boolean function.

Finally, we establish the following result.

\noindent
\textbf{Theorem}~\ref{THEO1} \textit{Let $(\cX,\le)$ be a finite lattice. 
    The class of $d$-monotone functions $f:\cX\to\{0,1\}$, that are represented in the form $f=F(g_1,g_2,\ldots,g_d)$, where $F$ is any Boolean function $F:\{0,1\}^d\to\{0,1\}$ and $g_1,\ldots,g_d:\cX\to \{0,1\}$ are any monotone functions, is learnable in time $\sigma(\cX)\cdot (\size(f)/d+1)^{d}$ where $\sigma(\cX)$ is the maximum sum of the number of immediate predecessors in a chain from the largest element to the smallest element in the lattice $\cX$ and  $\size(f)=\size(g_1)+\cdots+\size(g_d)$, where $\size(g_i)$ is the number of minimal elements in $g_i^{-1}(1)$. 
\\
    The algorithm asks at most $(\size(f)/d+1)^{d}$ equivalence queries and $\sigma(\cX)\cdot (\size(f)/d+1)^{d}$ membership queries.}

For the lattice $\{0,1\}^n$ with the standard $\le$, we have

\noindent
\textbf{Corollary}~\ref{Coroll}
    \textit{The class of $d$-monotone functions $f:\{0,1\}^n\to \{0,1\}$ that are represented in the form $f=F(g_1,g_2,\ldots,g_d)$, where $F$ is any Boolean function and $g_1,\ldots,g_d$ are any monotone DNF, is learnable in time $O(n^2)\cdot (\size(f)/d+1)^{d}$, where $\size(f)=\size(g_1)+\cdots+\size(g_d)$.  
\\
    The algorithm asks at most $(\size(f)/d+1)^{d}$ equivalence queries and $n^2\cdot (\size(f)/d+1)^{d}$ membership queries.}

\subsection{Consistent Hypothesis}
In this section, we give a procedure \textbf{Consistent} that receives $d$ and $\cX_0,\cX_1\subseteq \cX$ such that there is a $d$-monotone function $f$ that satisfies $f(x)=0$ for all $x\in \cX_0$ and $f(x)=1$ for all $x\in \cX_1$. The procedure returns a hypothesis $h$ that is a $d$-monotone function consistent with $f$ on $\cX_0\cup\cX_1$. That is, $h(x)=f(x)$ for all $x\in \cX_0\cup\cX_1$.

To establish the correctness and analyze the algorithm's complexity, we first prove two lemmas.
\begin{lemma}\label{PROM}
     Let $\cX_0,\cX_1\subseteq \cX$. Suppose there exists a $d$-monotone function $f$ such that $f(x)=0$ for all $x\in \cX_0$ and $f(x)=1$ for all $x\in \cX_1$. \textbf{Consistent}$(d,\cX_0,\cX_1)$ runs in polynomial time and constructs a $d$-monotone function $h$ of size $O(|\cX_0|+|\cX_1|)$ that is consistent with $f$ on $\cX_0\cup\cX_1$. 
\end{lemma}
\begin{proof}
    Consider the algorithm \textbf{Consistent} in Algorithm~\ref{Algo1n}. We prove the correctness by induction on $d$. 

    For $d=1$, the function $f$ is monotone. Suppose there is a monotone function such that $f(x)=0$ for $x\in \cX_0$ and $f(x)=1$ for $x\in \cX_1$. Then, there is no $z\in \cX_0$ and  $y\in \cX_1$ such that $z> y$. 

    In the first iteration, the procedure defines $F_1=\vee_{a\in \Min(\cX_1)}M_a$ and outputs $h=F_1$.
    If $z\in \cX_1$, then there is $a\le z$ such that $a\in\Min(\cX_1)$. Thus, $M_a(z)=1$ and consequently $h(z)=1$.
    If $z\in \cX_0$, there is no $y\in \cX_1$ such that $z>y$. Therefore, $M_a(z)=0$ for all $a\in \Min(\cX_1)$, and consequently $h(z)=0$.   

    Assume the statement is true for $(d-1)$-monotone functions. We now prove it for $d$-monotone functions. Let $f$ be a $d$-monotone function. In the first iteration of the procedure, it defines $\mathcal{S}_0 = \mathcal{X}_0$, $ \mathcal{S}_1 = \mathcal{X}_1$, $W_1=\Min(\cS_1)$, $F_1(x)=\vee_{a\in W_1}M_a(x)$, and $W_0=\{x\in \cS_0| F_1(x)=0\}$.  
    After the first iteration, it runs with the new points $\cS_1':=\cS_0\backslash W_0$ and $\cS_0':=\cS_1\cup W_0$.
    
    We first show that there is a $(d-1)$-monotone function $g$ such that $g(x)=0$ for all $x\in \cS_0'=\cS_1\cup W_0$ and $g(x)=1$ for all $x\in \cS_1'=\cS_0\backslash W_0$.

    Assume to the contrary that any function $g$ that is $0$ in $\cS_0'=\cS_1\cup W_0$ and $1$ in $\cS_1'=\cS_0\backslash W_0$ is $d'$-monotone for some $d'\ge d$, and is not $(d-1)$-monotone. 
    Let $\perp<x_1<x_2<\cdots<x_t$ be any chain where the function $g$ changes its value $d$ times. 
    Suppose the changes happen in $x_{i_1}<x_{i_2}<\cdots<x_{i_d}$. Since $g(\perp)=0$, we have $g(x_{i_1})=1$ and $g(x_{i_j})=(j\mod 2)$. Since $g(x_{i_1})=1$, we have $x_{i_1}\in \cS_1'= \cS_0\backslash W_0$. 
    Therefore $f(x_{i_1})=0$. Since $x_{i_1}\in \cS_0$ and $x_{i_1}\not\in W_0$, we have $F_1(x_{i_1})=1$, and therefore, there is $x_{0}\le x_{i_1}$ such that $x_{0}\in W_1=\Min(\cS_1)$. 
    In particular, $f(x_0)=1$. Since $f(x_0)=1$ and $f(x_{i_1})=0$, we have $x_0\not=x_{i_1}$ and therefore $x_0<x_{i_1}$. 
    
    Let $j\ge 2$. Since $x_{i_j}>x_{i_1}>x_{0}$, we have $F_1(x_{i_j})=1$ and therefore $x_{i_j}\not\in W_0$. Thus, $g(x_{i_j})=\overline{f(x_{i_j})}$ and $f(x_{i_j})=\overline{g(x_{i_j})}=(j-1\mod 2)$ for all $j\ge 2$.
    Hence, $\perp<x_0<x_{i_1}<x_{i_2}<\cdots<x_{i_d}$ is a chain for which $f$ changes its value along it $(d+1)$ times. This implies that $f$ is $d''$-monotone for some $d''\ge d+1$, which is a contradiction. 

    Now, by the induction hypothesis, $g=F_2\oplus F_3\oplus \cdots\oplus F_d$ satisfies $g(x)=0$ for every $x\in \cS_1\cup W_0$ and $g(x)=1$ for every $x\in \cS_0\backslash W_0$. We now show that $h=F_1\oplus g$ is the desired hypothesis. By the definition of $W_0$, if $x\in \cS_0\backslash W_0$, then $F_1(x)=1$ and $g(x)=1$, and therefore $h(x)=0$. If $x\in W_0$, then $F_1(x)=0$ and $g(x)=0$, and therefore $h(x)=0$. If $x\in \cS_1$, then $F_1(x)=1$ and $g(x)=0$, and therefore $h(x)=1$. \qed
\end{proof}

\begin{algorithm}
\caption{{\bf Consistent}$(d,\cX_0,\cX_1)$}\label{Algo1n}
\begin{algorithmic}[1]
    \STATE Let $\mathcal{S}_0 = \mathcal{X}_0; \mathcal{S}_1 = \mathcal{X}_1$.
    \FOR{$i = 1$ to $d$}
        \STATE Let $W_1 \gets \Min (\mathcal{S}_1)$.
        \STATE\label{mmm} Define $F_i = \bigvee_{a \in W_1} M_a$ \ \ \ \ $\backslash *$If $W_1=\emptyset$ then $F_i=0$
        \STATE $W_0\gets \{x \in \mathcal{S}_{0} \mid F_i(x) = 0\}$
        \STATE $\cS_1\gets (\cS_0\backslash W_0).$
        \STATE $\cS_0\gets \cS_1\cup W_0.$
    \ENDFOR
    \STATE Output $h = F_1 \oplus F_2 \oplus \cdots \oplus F_d.$
\end{algorithmic}
\end{algorithm}

In~\cite{TakimotoSM00} (page 16), Takimoto et al. claim that if $f=g_1\oplus g_2\oplus\cdots\oplus g_d$, where $g_i$ is monotone for every $i\le d$, $g_{i+1}\not=g_i$, and $g_{i+1}\Rightarrow g_i$ for every $i\le d-1$, then $g_i=\cM(f_i)$. In the appendix, we show that this claim is not entirely accurate. The following lemma outlines the conditions under which this statement holds.

\begin{lemma}\label{GGG}
    If $f=g_1\oplus \cdots\oplus g_d$, where $g_i$ is monotone function for every $i\le d$, $g_{i+1}\Rightarrow g_i$ and $\Min(g_{i+1})\cap \Min(g_i)=\emptyset$ for every $i\le d-1$, then $\cM(f_i)=g_i$.
\end{lemma}
\begin{proof}
    It is enough to prove that $\cM(f_1)=g_1$. This is because if we prove that $\cM(f_1)=g_1$, then
    $$f_2=f_1\oplus\cM(f_1)=f\oplus \cM(f_1)=(g_1\oplus g_2\oplus\cdots\oplus g_d)\oplus g_1=g_2\oplus\cdots\oplus g_d,$$ and therefore $\cM(f_2)=g_2$. Then, by induction, the result follows.

    Recall that $f_1=f$. We first prove that $\cM(f)\Rightarrow g_1$. We show that $\Min(\cM(f))\subset \Min(g_1)$. Let $a\in \Min(\cM(f))=\Min(f)$. 
    Then $f(a)=1$ and for every $b<a$, we have $f(b)=0$. We now show that $g_1(a)=1$ and $g_i(a)=0$ for all $i>1$. If $g_i(a)=0$ for all $i$, then $f(a)=0$, and we get a contradiction. 
    
    If $g_i(a)=1$ for some $i>1$, then $g_2(a)=1$ and there is $a'\in \Min(g_2)$, $a'\le a$, such that $g_2(a')=1$. Then $g_1(a')=1$, and since $\Min(g_1)\cap\Min(g_2)=\emptyset$, there is a $a''\in \Min(g_1)$ such that $a''<a'$ and $g_1(a'')=1$. Since $a''<a'$ and $a'\in \Min(g_2)$, we have $g_2(a'')=0$ and therefore $g_i(a'')=0$ for all $i>1$. Therefore, $f(a'')=g_1(a'')=1$. 
    Since $a''<a'\le a\in \Min(f)$, we have $f(a'')=0$, which is a contradiction. 
    Therefore $g_1(a)=1$ and $g_i(a)=0$ for all $i>1$. Since for every $b<a$, $f(b)=0$, we have for every $b<a$, $g_i(b)=0$ for all $i$. This implies that $a\in \Min(g_1)$. 

    We now prove that $g_1\Rightarrow \cM(f)$. Let $a\in \Min(g_1)$. Then $g_1(a)=1$ and for every $b<a$, we have $g_1(b)=0$. Therefore, for every $b<a$ and every $i>1$, we have $g_i(b)=0$. If $g_i(a)=1$ for some $i>1$, then $g_2(a)=1$. Then $a\in \Min(g_2)$, and since $\Min(g_1)\cap\Min(g_2)=\emptyset$, we get a contradiction. Therefore, $g_1(a)=1$, $g_i(a)=0$ for all $i>1$ and for every $b<a$, $g_j(b)=0$ for all $j\ge 1$. Therefore, $f(a)=1$ and for every $b<a$, $f(b)=0$. Thus, $a\in \Min(f)=\Min(\cM(f))$. \qed
\end{proof}

The following lemma proves that the output $F_1\oplus \cdots\oplus F_d$ of the procedure \textbf{Consistent} is the strict monotone representation of $h$.
\begin{lemma}\label{FiMhi}
     Let $\cX_0,\cX_1\subseteq \cX$. Suppose there is a $d$-monotone function $f$ such that $f(x)=0$ for all $x\in \cX_0$ and $f(x)=1$ for all $x\in \cX_1$. Let $h=F_1\oplus\cdots\oplus F_d$ be the output of \textbf{Consistent}$(d,\cX_0,\cX_1)$. Then $F_i=\cM(h_i)$.
\end{lemma}
\begin{proof}
    We use Lemma~\ref{GGG}. By step~\ref{mmm} in the procedure \textbf{Consistent}, we have that each $F_i$ is a monotone function. Now, it is enough to prove that $\Min(F_1)\cap \Min(F_2)=\emptyset$ and $F_2\Rightarrow F_1$. Then, the result follows by induction. 

    Since $\Min(F_1)=\Min(\cS_1)=\Min(\cX_1)\subseteq \cX_1$ and $\Min(F_2)=\Min(\cS_0\backslash W_0)\subseteq \cX_0$, we have $\Min(F_1)\cap \Min(F_2)=\emptyset$.

    Now if $F_2(z)=1$, then since $F_2=\vee_{a\in \Min(\cS_0\backslash W_0)}M_a$ and $\Min(\cS_0\backslash W_0)=\Min(\cX_0\backslash \{x\in \cX_0|F_1(x)=0\})$, there is an $a\in \cX_0\backslash \{x\in \cX_0|F_1(x)=0\}$ such that $a\le z$. Then $F_1(a)=1$ and since $F_1$ monotone and $z\ge a$, we have $F_1(z)=1$. Therefore, $F_2\Rightarrow F_1$.\qed
\end{proof}

\subsection{The Main Algorithm}
In this section, we present the algorithm and prove Theorem~\ref{THEO1} and Corollary~\ref{Coroll}.  

We first prove two lemmas needed to establish the correctness and determine the complexity of the algorithm. The first is:
\begin{lemma}\label{Bin}
    Let $g_1,\ldots,g_d$ be monotone functions. Let $h$ be a monotone function such that 
    \begin{eqnarray}\label{Bin1}
     \Min(h)\subseteq \bigcup_{J\subseteq [d]}\bigvee_{i\in J}\Min(g_i).   
    \end{eqnarray}
    For any $I\subseteq [d]$, we have
    $$\Min\left(h\wedge\bigwedge_{i\in I}g_i\right) \subseteq \bigcup_{J\subseteq [d]}\bigvee_{i\in J}\Min(g_i).$$
\end{lemma}
\begin{proof}
Let $a\in \cX$. Recall that $M_a:\cX\to\{0,1\}$ is the function that $M_a(x)=1$ if and only if $x\ge a$. 

Let $a$ be a minimal element of $h\wedge \wedge_{i\in I}g_i$. Let $\Min(h)=\{u_1,\ldots\al,u_t\}$. 
    Then, $h=M_{u_1}\vee M_{u_2}\vee \cdots \vee M_{u_t}$ and 
    $$h\wedge\wedge_{i\in I}g_i= (M_{u_1}\wedge\wedge_{i\in I}g_i)\vee (M_{u_2}\wedge\wedge_{i\in I}g_i)\vee \cdots \vee (M_{u_t}\wedge\wedge_{i\in I}g_j).$$
    By item~\ref{MinORAND1} Lemma~\ref{MinORAND}, $a$ is a minimal element of some $M_{u_\ell}\wedge \wedge_{i\in I}g_i$.
    
    Now, by (\ref{Bin1}), there is $J_\ell\subseteq [d]$ such that $u_\ell=\vee_{j\in J_\ell}u_{\ell,j}$ where $u_{\ell,j}\in \Min(g_j)$. 
    Therefore, by item~\ref{MinORAND3} in Lemma~\ref{MinORAND}, $M_{u_\ell}=\wedge_{j\in J_\ell}M_{u_{\ell,j}}$ where $M_{u_{\ell,j}}$ is a minterm in $g_j$. Since $M_{u_{\ell,j}}\Rightarrow g_j$, $M_{u_{\ell,j}}\wedge g_j=M_{u_{\ell,j}}$. 
    Therefore, $M_{u_\ell}\wedge \wedge_{i\in I}g_i=\wedge_{j\in J_\ell}M_{u_{\ell,j}}\wedge \wedge_{i\in I\Delta J_\ell}g_i$. 
    
    Thus, by item~\ref{MinORAND2} in Lemma~\ref{MinORAND},
    $$a\in \Min(\wedge_{j\in J_\ell}M_{u_{\ell,j}}\wedge \wedge_{i\in I\Delta J_\ell}g_i)\subseteq \bigvee_{j\in I\cup J_\ell}\Min(g_j).$$ \qed   
\end{proof}

The second lemma is given below.
\begin{lemma}\label{FTAlg}
    Let $f=F(g_1,\ldots,g_d)$ where $F:\{0,1\}^d\to \{0,1\}$ and $g_1,\ldots,g_d$ are monotone functions. Let $h$ be a $d$-monotone function such that 
    \begin{eqnarray}\label{Lasst}
        \bigcup_{i=1}^d \Min(\cM(h_i))\subseteq \bigcup_{J\subseteq [d]}\bigvee_{j\in J}\Min(g_j).
    \end{eqnarray}
    Then
    $$\min(f\oplus h)\subseteq \left(\bigcup_{J\subseteq [d]}\bigvee_{j\in J} \Min(g_j)\right)
    $$
\end{lemma}
\begin{proof}
    Consider
    $$G=f\oplus h=F(g_1,\ldots,g_d)\oplus \cM(h_1)\oplus\cdots\oplus\cM(h_d).$$
    Let $a\in \min(G)$ be a local minimal element of $G$. Then $G(a)=1$ and for every immediate predecessor $b$ of $a$ we have $G(b)=0$. Suppose $g_i(a)=1$ for all $i\in I$, $g_i(a)=0$ for all $i\not\in I$, $\cM(h_1)(a)=\cdots=\cM(h_\ell)(a)=1$, and $\cM(h_{\ell+1})(a)=\cdots=\cM(h_d)(a)=0$. 
    Since $g_i$ and $\cM(h_j)$ are monotone functions, for every immediate predecessor $b$ of $a$ we have $g_i(b)=0$ for all $i\not\in I$ and $\cM(h_{\ell+1})(b)=\cdots=\cM(h_d)(b)=0$. 
    Since $f(a)\not=f(b)$, either $\cM(h_\ell)(b)=0$ or $g_i(b)=0$ for some $i\in I$. 
    Therefore, $a$ is a local minimal element of $H:=\cM(h_\ell)\wedge \wedge_{i\in I}g_i$ for some $\ell\in [d]$ and $I\subseteq [d]$. 
    Since $H$ is monotone, $\min(H)=\Min(H)$ and therefore
    \begin{eqnarray}\label{Lasst2}
        a\in \Min\left(\cM(h_\ell)\wedge \bigwedge_{i\in I}g_i\right).
    \end{eqnarray}
    By (\ref{Lasst}), (\ref{Lasst2}) and Lemma~\ref{Bin}.    $$a\in \bigcup_{J\subseteq [d]}\bigvee_{j\in J} \Min(g_j).$$\qed
\end{proof}

We now give the proof of the main Theorem. Consider the algorithm \textbf{Learn $d$-Monotone} in Algorithm~\ref{Alg2}. The following proves Theorem~\ref{THEO1}.

\begin{algorithm}
\caption{\textbf{Learn $d$-Monotone}}\label{LearndM}
\begin{algorithmic}[1]
    \STATE $\mathcal{X}_0 = \mathcal{X}_1 = \emptyset$
    \STATE $h\gets 0$
    \WHILE{$\text{EQ}(h) \neq \text{YES}$}
        \STATE Let $a$ be a counterexample
        \WHILE{there is an immediate predecessor $b$ of $a$ such that $h(b) \neq f(b)$}\label{step5}
            \STATE $a \gets b$
        \ENDWHILE
        \IF{$f(a) = 1$}
            \STATE $\mathcal{X}_1 \gets \mathcal{X}_1 \cup \{a\}$
        \ELSE
            \STATE $\mathcal{X}_0 \gets \mathcal{X}_0 \cup \{a\}$
        \ENDIF
            \STATE $h \gets \textbf{Consistent}(d, \mathcal{X}_0, \mathcal{X}_1)$
    \ENDWHILE
    \STATE Output $h$
\end{algorithmic}\label{Alg2}
\end{algorithm}

\begin{theorem}
    Algorithm {\bf Learn $d$-Monotone} learns $d$-monotone functions $f$ with at most $R(f)$ equivalence queries and $R(f)\sigma(\cX)$ membership queries, where
    $$R(f)= \left|\bigcup_{I\subseteq [d]}\bigvee_{i\in I}\Min(g_i) \right|\le \left(\frac{\size(f)}{d}+1\right)^d.$$
\end{theorem}
\begin{proof}
    Let $f=F(g_1,g_2,\ldots,g_d)$ be the target function. We will show by induction that at the end of iteration~$t$, the sets $\cX_0$, $\cX_1$ and the hypothesis $h$ satisfy:
    \begin{eqnarray}\label{LasstE1}
        \cX_0\cup\cX_1\subseteq \bigcup_{I\subseteq [d]}\bigvee_{i\in I}\Min(g_i)
    \end{eqnarray}
    \begin{eqnarray}\label{LasstE2}
        \mbox{For every}\ u\in\cX_0\cup\cX_1 \mbox{\ \ we have $f(u)=h(u)$}
    \end{eqnarray}
    and 
    \begin{eqnarray}\label{LasstE3}
        \left|\cX_0\cup\cX_1\right|=t.
    \end{eqnarray}
    At the first iteration, we have $h=0$. The equivalence query returns $a'$ such that $f(a')=1$. Then, the algorithm in step~\ref{step5} finds a local minimal element $a$ of $f$ and adds it to $\cX_0$ or $\cX_1$. Therefore, at the end of the first iteration, by Lemma~\ref{BaseF}, (\ref{LasstE1}) holds. By Lemma~\ref{PROM}, (\ref{LasstE2}) holds. Also, (\ref{LasstE3}) holds since $|\cX_0\cup\cX_1|=|\{a\}|=1$. 

    Now suppose (\ref{LasstE1})-(\ref{LasstE3}) hold at the end of iteration $t$. We prove that they hold at the end of iteration~$t+1$.
    
    At iteration $t+1$, if EQ$(h)$ returns a counterexample $a'$, then $f(a')\not=h(a')$ and therefore $f(a')\oplus h(a')=1$. In step~\ref{step5} of the algorithm, it continues to go down in the lattice until it finds an $a$ such that $f(a)\oplus h(a)=1$ and for every immediate predecessor $b$ of $a$, $f(b)\oplus h(b)=0$. Such an $a$ exists because $f(\perp)\oplus h(\perp)=0$. Therefore, $a\in \min(f\oplus h)$. By Lemma~\ref{FTAlg}, we have,
    $$a\in \left(\bigcup_{I\subseteq [d]}\bigvee_{i\in I} \Min(g_i)\right).$$
    By the induction hypothesis (\ref{LasstE2}), $f(u)=h(u)$ for all $u\in \cX_0\cup\cX_1$. Since $f(a)\not=h(a)$, we have $a\not\in \cX_0\cup\cX_1$ and since $a$ is added either to $\cX_0$ or $\cX_1$, at iteration $t+1$,  (\ref{LasstE1}) holds and (\ref{LasstE3}) holds at the end of iteration $t+1$. Now, (\ref{LasstE2}) also holds because $a$ is added to $\cX_1$ if $f(a)=1$ and to $\cX_0$ if $f(a)=0$ and by Lemma~\ref{PROM}, $f(u)=h(u)$ for all $u\in \cX_0\cup\cX_1\cup\{a\}$.

    This completes the proof of (\ref{LasstE1})-(\ref{LasstE3}).

    Since $\size(f)=\size(g_1)+\cdots+\size(g_d)$, and after each equivalence query, the algorithm adds an element either to $\cX_0$ or $\cX_1$, and by (\ref{LasstE1}), the number of equivalence queries is at most
    \begin{eqnarray*}
        \left|     \bigcup_{I\subseteq [d]}\bigvee_{i\in I}\Min\left(g_i\right)\right|&\le&
     \prod_{i=1}^d(\size(g_i)+1)-1\\  &\le&\left(\frac{\size(f)}{d}+1\right)^d=R(f). \ \ \ \ \ \text{AM-GM Inequality}
    \end{eqnarray*}
   After each equivalence query, the algorithm asks membership queries to go down in the lattice. The worst-case number of membership queries after each equivalence query is $\sigma(\cX)$. Therefore, the number of membership queries that the algorithm asks is at most $\sigma(\cX)R(f)$.\qed
\end{proof}

\section{Strict Monotone Representation Size}
In this section, we compare the size of the strict monotone representation of $f$ with the size of $f$ using the representation presented in this paper. We show that there exists a $d$-monotone Boolean function $f$ with $\size(f)=s$ that has size $\Omega((s/d)^d)$ in the strict monotone representation. We also show that this is a tight bound.

Throughout this section, the lattice is $\{0,1\}^n$ with the standard $\le$. 

First, by Lemma~\ref{Rep01} and Lemma~\ref{MaimM}, we have the following:
\begin{lemma}
    $f:\cX\to \{0,1\}$ is $d$-monotone if and only if $f=\cM(f_1)\oplus \cM(f_2)\oplus \cdots \oplus \cM(f_d)$. 
\end{lemma}
We now define two classes of $d$-monotone functions.
\begin{enumerate}
    \item The class $d$-M \mn{$d$-M} is the class of $d$-monotone functions $f$ that are represented as $f=F(g_1,\ldots,g_d)$ where $F:\{0,1\}^d\to \{0,1\}$ is any Boolean function such that $F(0^d)=0$ and $g_1,g_2,\ldots,g_d:\cX\to \{0,1\}$ are any monotone functions. 
    \item The class $d$-M$(\oplus\cM)$ \mn{$d$-M$(\oplus\cM)$} is the class of $d$-monotone functions $f$ represented in the strict monotone representation $f=\cM(f_1)\oplus \cM(f_2)\oplus \cdots \oplus \cM(f_d)$.
\end{enumerate}
We define $\size(f)$ to be the minimum possible $\size(g_1)+\cdots+\size(g_d)$ of representations of $f=F(g_1,\ldots,g_d)$ in $d$-M. We define $\size_{\oplus\cM}(f)=\size(\cM(f_1))+\cdots+\size(\cM(f_d))$\mn{$\size_{\oplus\cM}(f)$}. 

Before proving the relationship between $\text{size}(f)$ and $\text{size}_{\oplus\mathcal{M}}(f)$, we present two lemmas that will be used to establish this relationship.

\begin{lemma}\label{fFgd22}
    Let $f=F(g_1,\ldots,g_d)$, where $g_i$ are monotone functions and $F(0^d)=0$. Then, for every $k$ 
    $$\bigcup_{k=1}^d \Min(\cM(f_{k}))\subseteq \bigcup_{I\subseteq [d]}\bigvee_{i\in I}\Min\left(g_i\right).$$
\end{lemma}
\begin{proof} By Lemma~\ref{FiMhi}, every hypothesis $h=F_1\oplus\cdots\oplus F_d$ in the algorithm \textbf{Learn $d$-Monotone}~\ref{LearndM}  satisfies $F_i=\cM(h_i)$. Since the final hypothesis of the algorithm is $f$, the final output of the algorithm is $F_1'\oplus F_2'\oplus\cdots\oplus F_d'$ where $F_i'=\cM(f_i)$. In the procedure \textbf{Consistent}~\ref{Algo1n}, the minimal elements of all $F_i'=\cM(f_i)$ are from $\cX_0\cup\cX_1$, and by (\ref{LasstE1}), we have
\begin{eqnarray*}
        \cX_0\cup\cX_1\subseteq \bigcup_{I\subseteq [d]}\bigvee_{i\in I}\Min(g_i).
    \end{eqnarray*}\qed
\end{proof}

We now prove
\begin{lemma}\label{SMS}
    We have
    $$\size_{\oplus \cM}(f)\le \left(\frac{\size(f)}{d}+1\right)^{d}-1.$$
\end{lemma}
\begin{proof}
    Let $f=F(g_1,\ldots,g_d)$ where $F:\{0,1\}^d\to \{0,1\}$ and $F(0^d)=0$. Suppose $s_i=\size(g_i)$. By Lemma~\ref{fFgd22}, we have
        $$\bigcup_{k=1}^d\Min(\cM(f_{k}))\subseteq \bigcup_{I\subseteq [d]}\bigvee_{i\in I}\Min\left(g_i\right).$$
    Therefore, by item~\ref{Fi3} in Lemma~\ref{Fi} and the AM-GM inequality,
    \begin{eqnarray*}      \size_{\oplus\cM}(f)&=& \sum_{k=1}^d\size(\cM(f_k))=\sum_{k=1}^d|\Min(\cM(f_k))|\\   &=&\left|\bigcup_{k=1}^d\Min(\cM(f_{k}))\right|
    \le \left|     \bigcup_{I\subseteq [d]}\bigvee_{i\in I}\Min\left(g_i\right)\right|\\
        &\le& \prod_{i=1}^d(\size(g_i)+1)-1  \le\left(\frac{\size(f)}{d}+1\right)^d-1.
    \end{eqnarray*}\qed
\end{proof}
We now show that this bound is tight.
\begin{lemma}
There is a $d$-monotone function $f$ such that $$\size_{\oplus \cM}(f)= \left(\frac{\size(f)}{d}+1\right)^{d}-1.$$
\end{lemma}
\begin{proof}
    Consider the function 
    $f=y_1\oplus \cdots\oplus y_d$ where $y_i=x_{i,1}\vee\cdots\vee x_{i,t}$ where $t=n/d$. The size of~$f$ is $d(n/d)=n$.

    First
    $$y_1\oplus\cdots\oplus y_d= G_1\oplus G_2\oplus \cdots\oplus G_d$$
    where
    $$G_k=\bigvee_{1\le i_1<i_2<\cdots<i_k\le d}\left(\bigwedge_{j=1}^k y_{i_j}\right).$$
    This is because if $\ell$ of the functions $y_i$ are equal to $1$ then $G_1=G_2=\cdots=G_\ell=1$ and $G_{\ell+1}=\cdots=G_d=0$. 

    Since $G_d\Rightarrow G_{d-1}\Rightarrow \cdots\Rightarrow G_1$ and $\Min(G_i)\cap \Min(G_{i+1})=\emptyset$, by Lemma~\ref{GGG}, we have $G_i=\cM(f_i)$. Now
    \begin{eqnarray*}
      \size_{\oplus\cM}(f)&=&\size(G_1)+\size(G_2)+\cdots+\size(G_d)\\
      &=&dt+\binom{d}{2}t^2+\cdots+\binom{d}{d}t^d\\
      &=&(t+1)^d-1=\left(\frac{\size(f)}{d}+1\right)^{d}-1.
    \end{eqnarray*}\qed
\end{proof}
\bibliographystyle{splncs04}
\bibliography{dMonotone}

\begin{thebibliography}{10}
\providecommand{\url}[1]{\texttt{#1}}
\providecommand{\urlprefix}{URL }
\providecommand{\doi}[1]{https://doi.org/#1}

\bibitem{AmanoM06}
Amano, K., Maruoka, A.: On learning monotone boolean functions under the uniform distribution. Theor. Comput. Sci.  \textbf{350}(1),  3--12 (2006). \doi{10.1016/J.TCS.2005.10.012}, \url{https://doi.org/10.1016/j.tcs.2005.10.012}

\bibitem{Angluin87}
Angluin, D.: Queries and concept learning. Mach. Learn.  \textbf{2}(4),  319--342 (1987). \doi{10.1007/BF00116828}, \url{https://doi.org/10.1007/BF00116828}

\bibitem{Black}
Black, H.: Nearly optimal bounds for sample-based testing and learning of {\textdollar}k{\textdollar}-monotone functions. CoRR  \textbf{abs/2310.12375} (2023). \doi{10.48550/ARXIV.2310.12375}, \url{https://doi.org/10.48550/arXiv.2310.12375}

\bibitem{BlaisCOST15}
Blais, E., Canonne, C.L., Oliveira, I.C., Servedio, R.A., Tan, L.: Learning circuits with few negations. In: Garg, N., Jansen, K., Rao, A., Rolim, J.D.P. (eds.) Approximation, Randomization, and Combinatorial Optimization. Algorithms and Techniques, {APPROX/RANDOM} 2015, August 24-26, 2015, Princeton, NJ, {USA}. LIPIcs, vol.~40, pp. 512--527. Schloss Dagstuhl - Leibniz-Zentrum f{\"{u}}r Informatik (2015). \doi{10.4230/LIPICS.APPROX-RANDOM.2015.512}, \url{https://doi.org/10.4230/LIPIcs.APPROX-RANDOM.2015.512}

\bibitem{BlumBL98}
Blum, A., Burch, C., Langford, J.: On learning monotone boolean functions. In: 39th Annual Symposium on Foundations of Computer Science, {FOCS} '98, November 8-11, 1998, Palo Alto, California, {USA}. pp. 408--415. {IEEE} Computer Society (1998). \doi{10.1109/SFCS.1998.743491}, \url{https://doi.org/10.1109/SFCS.1998.743491}

\bibitem{Bshouty95}
Bshouty, N.H.: Exact learning boolean function via the monotone theory. Inf. Comput.  \textbf{123}(1),  146--153 (1995). \doi{10.1006/INCO.1995.1164}, \url{https://doi.org/10.1006/inco.1995.1164}

\bibitem{Bshouty97}
Bshouty, N.H.: Simple learning algorithms using divide and conquer. Comput. Complex.  \textbf{6}(2),  174--194 (1997). \doi{10.1007/BF01262930}, \url{https://doi.org/10.1007/BF01262930}

\bibitem{BshoutyT96}
Bshouty, N.H., Tamon, C.: On the fourier spectrum of monotone functions. J. {ACM}  \textbf{43}(4),  747--770 (1996). \doi{10.1145/234533.234564}, \url{https://doi.org/10.1145/234533.234564}

\bibitem{GoldreichGLRS00}
Goldreich, O., Goldwasser, S., Lehman, E., Ron, D., Samorodnitsky, A.: Testing monotonicity. Comb.  \textbf{20}(3),  301--337 (2000). \doi{10.1007/S004930070011}, \url{https://doi.org/10.1007/s004930070011}

\bibitem{GuijarroLR01}
Guijarro, D., Lav{\'{\i}}n, V., Raghavan, V.: Monotone term decision lists. Theor. Comput. Sci.  \textbf{259}(1-2),  549--575 (2001). \doi{10.1016/S0304-3975(00)00043-8}, \url{https://doi.org/10.1016/S0304-3975(00)00043-8}

\bibitem{HarmsY22}
Harms, N., Yoshida, Y.: Downsampling for testing and learning in product distributions. In: Bojanczyk, M., Merelli, E., Woodruff, D.P. (eds.) 49th International Colloquium on Automata, Languages, and Programming, {ICALP} 2022, July 4-8, 2022, Paris, France. LIPIcs, vol.~229, pp. 71:1--71:19. Schloss Dagstuhl - Leibniz-Zentrum f{\"{u}}r Informatik (2022). \doi{10.4230/LIPICS.ICALP.2022.71}, \url{https://doi.org/10.4230/LIPIcs.ICALP.2022.71}

\bibitem{LangeRV22}
Lange, J., Rubinfeld, R., Vasilyan, A.: Properly learning monotone functions via local correction. In: 63rd {IEEE} Annual Symposium on Foundations of Computer Science, {FOCS} 2022, Denver, CO, USA, October 31 - November 3, 2022. pp. 75--86. {IEEE} (2022). \doi{10.1109/FOCS54457.2022.00015}, \url{https://doi.org/10.1109/FOCS54457.2022.00015}

\bibitem{LangeV23}
Lange, J., Vasilyan, A.: Agnostic proper learning of monotone functions: beyond the black-box correction barrier. In: 64th {IEEE} Annual Symposium on Foundations of Computer Science, {FOCS} 2023, Santa Cruz, CA, USA, November 6-9, 2023. pp. 1149--1170. {IEEE} (2023). \doi{10.1109/FOCS57990.2023.00068}, \url{https://doi.org/10.1109/FOCS57990.2023.00068}

\bibitem{Markov58}
Markov, A.A.: On the inversion complexity of a system of functions. J. {ACM}  \textbf{5}(4),  331--334 (1958). \doi{10.1145/320941.320945}, \url{https://doi.org/10.1145/320941.320945}

\bibitem{ODonnellS07}
O'Donnell, R., Servedio, R.A.: Learning monotone decision trees in polynomial time. {SIAM} J. Comput.  \textbf{37}(3),  827--844 (2007). \doi{10.1137/060669309}, \url{https://doi.org/10.1137/060669309}

\bibitem{ODonnellW09}
O'Donnell, R., Wimmer, K.: Kkl, kruskal-katona, and monotone nets. In: 50th Annual {IEEE} Symposium on Foundations of Computer Science, {FOCS} 2009, October 25-27, 2009, Atlanta, Georgia, {USA}. pp. 725--734. {IEEE} Computer Society (2009). \doi{10.1109/FOCS.2009.78}, \url{https://doi.org/10.1109/FOCS.2009.78}

\bibitem{Servedio04}
Servedio, R.A.: On learning monotone {DNF} under product distributions. Inf. Comput.  \textbf{193}(1),  57--74 (2004). \doi{10.1016/J.IC.2004.04.003}, \url{https://doi.org/10.1016/j.ic.2004.04.003}

\bibitem{TakimotoSM00}
Takimoto, E., Sakai, Y., Maruoka, A.: The learnability of exclusive-or expansions based on monotone {DNF} formulas. Theor. Comput. Sci.  \textbf{241}(1-2),  37--50 (2000). \doi{10.1016/S0304-3975(99)00265-0}, \url{https://doi.org/10.1016/S0304-3975(99)00265-0}

\end{thebibliography}

\newpage
\appendix
\section{Another Representation}
In~\cite{TakimotoSM00} (page 16), Takimoto et al. claim that if $f=g_1\oplus g_2\oplus\cdots\oplus g_d$, where each $g_i$ is monotone, and for every $i\le d-1$, $g_{i+1}\not=g_i$, and $g_{i+1}\Rightarrow g_i$, then $g_i=\cM(f_i)$. In this appendix, we show in Lemma~\ref{L19} that this claim is not entirely accurate. See also~\cite{GuijarroLR01} page~560. We show that there exists a function $f=g_1\oplus g_2\oplus\cdots\oplus g_d$ of size $s$, where each $g_i$ is monotone, and for every $i\le d-1$, $g_{i+1}\not=g_i$, and $g_{i+1}\Rightarrow g_i$, that satisfies
$$\size_{\oplus \cM}(f)=\Omega\left(\left(\frac{2s}{d^2}\right)^d\right).$$
This, in particular, implies that Takimoto et al.'s claim is not true.

We define $d$-M$(\oplus\cI)$ to be the class of all the $d$-monotone functions $f=g_1\oplus\cdots\oplus g_d$ where $g_d\Rightarrow g_{d-1}\Rightarrow\cdots\Rightarrow g_1$ and $g_{i+1}\not=g_i$ for all $i\le d-1$. Recall the class 
M$(\oplus\cM)$ of all the $d$-monotone functions with strict monotone representations.

We define
$\size_{\oplus\cI}(f)$ to be the minimum possible $\size(g_1)+\cdots+\size(g_d)$ of such representations. 

Throughout this appendix, the lattice is $\{0,1\}^n$ with the standard $\le$.

We first start with the following lemma.
\begin{lemma}\label{MintermsGener}
    Let $f:\{0,1\}^n\to \{0,1\}$ be a Boolean function such that $f(0^n)=0$ and there is $x^{(1)}<x^{(2)}<\cdots<x^{(m)}$ where $f(x^{(i)})=i\mod 2$, $x^{(i)}$ is an immediate predecessor of $x^{(i+1)}$, and\footnote{For $x\in \{0,1\}^n$, $\wt(x)$ denotes the Hamming weight of $x$, i.e., the number of ones in $x$.} $\wt(x^{(i)})=i$. Then $x^{(i)}$ is a minimal element of $\cM(f_i)$.
\end{lemma}
\begin{proof}
    Since $f(0^n)=0$, every element $a$ of weight $1$ that satisfies $f_1(a)=f(a)=1$ (including $x^{(1)}$) is in $\Min(f_1^{-1}(1))$ and therefore is a minimal element of $\cM(f_1)$, and $\cM(f_1)(a)=1$. Consider $f_2=f_1\oplus \cM(f_1)$. 
    Then $f_2(0^n)=0$ and $f_2$ is zero in every element of weight $1$. Since $x^{(i)}\ge x^{(1)}$, we have $\cM(f_1)(x_i)=1$ and therefore $f_2(x^{(i)})=f_1(x_i)\oplus \cM(f_1)(x_i)=((i+1)\mod 2)$. Then every element $a$ of weight $2$ that satisfies $f_2(a)=1$ (including $x^{(2)}$) is a minimal element of $\cM(f_2)$. By induction, the result follows. \qed
\end{proof}

\begin{lemma}\label{L19}
There exist a $d$-monotone function $f$ such that $$\size_{\oplus \cM}(f)=\Omega\left(\left(\frac{2\cdot \size_{\oplus \cI}(f)}{d^2}\right)^d\right).$$
\end{lemma}
\begin{proof}
    Consider the function $$f=(y_1\vee y_2\vee \cdots\vee y_d)\oplus (y_2\vee y_3\cdots \vee y_{d})\oplus \cdots\oplus (y_{d-1}\vee y_d)\oplus y_d,$$
    where $y_i=x_{i,1}\vee x_{i,2}\vee \cdots\vee x_{i,t}$ where $t=2n/(d(d+1))$. The number of variables in $f$ and the size of $f$ is $n$. 

    We will now use Lemma~\ref{MintermsGener}. For every $(1,j_1),(2,j_2),\ldots,(d,j_d)$ where $j_i\in [t]$ for all $i\in [d]$, consider the elements $x^{(1)}<\cdots<x^{(d)}$ where $x^{(\ell)}$ has $1$ in entries $(1,j_1),(2,j_2),\ldots,(\ell,j_\ell)$ and $0$ in the other entries. Then $x^{(i)}$, $i\in [d]$, satisfies the conditions in Lemma~\ref{MintermsGener}. Therefore, $x^{(d)}$ is a minimal element of $\cM(f_d)$. The number of such elements is $t^d=\Omega((2\cdot \size_{\oplus \cI}(f)/d^2)^d)$.\qed
\end{proof}

We note here that if we choose $y_i$ to be a disjunction of $n/(id)$ variables, then we get a slightly better lower bound $\sim (e\cdot\size_{\oplus\cI}(f)/d^2)^d$.

We now show.
\begin{lemma}
    We have
    $$\size_{\oplus \cM}(f)\le \left(\frac{\size_{\oplus \cI}(f)}{d}+1\right)^d.$$
\end{lemma}
\begin{proof}
Obviously, $\size(f)\le \size_{\oplus\cI}(f)$. Now the result follows from Lemma~\ref{SMS}.\qed  
\end{proof}
\end{document}